\documentclass[twocolumn,10pt]{article}
\usepackage[T1]{fontenc}
\usepackage[utf8]{inputenc}
\usepackage{mathptmx}
\usepackage[scaled=.90]{helvet}
\usepackage{courier}

\usepackage[a4paper,top=0.7in,bottom=0.9in,left=0.62in,right=0.62in,includefoot,footskip=0.35in]{geometry}
\usepackage{graphicx}
\usepackage{xcolor}
\usepackage{tikz}
\usetikzlibrary{arrows.meta,positioning,shapes.geometric,calc,fit,backgrounds}
\usepackage{cuted}
\usepackage{microtype}
\usepackage{parskip}
\usepackage{tcolorbox}
\tcbuselibrary{skins,breakable}
\usepackage{academicons}
\usepackage{fontawesome5}
\usepackage{hyperref}
\usepackage{xurl}
\usepackage{lastpage}
\usepackage{booktabs}
\usepackage{colortbl}
\usepackage{tabularx}
\usepackage{caption}
\usepackage{enumitem}
\usepackage{fancyhdr}
\usepackage{titlesec}
\usepackage{float}
\usepackage{amsmath, amssymb}
\usepackage{mathtools}
\usepackage{multirow}
\usepackage{amsthm}
\usepackage[ruled,vlined,linesnumbered]{algorithm2e}
\newtheorem{definition}{Definition}
\newtheorem{proposition}{Proposition}

\usepackage{cite}
\usepackage{stfloats}

\renewcommand{\arraystretch}{1.12}

\usepackage{wrapfig}
\graphicspath{{figs/}}

\definecolor{primary}{HTML}{003049}
\definecolor{accent}{HTML}{F77F00}
\definecolor{textgray}{HTML}{2F2F2F}
\newcommand{\rowsep}{\arrayrulecolor{black!20}\specialrule{0.2pt}{1.3pt}{1.3pt}\arrayrulecolor{black}}

\hypersetup{colorlinks=true, linkcolor=primary, citecolor=primary, urlcolor=primary}

\titleformat{\section}{\color{primary}\sffamily\large\bfseries}{\thesection}{0.8em}{}[\vspace{0.2em}\titlerule]
\titleformat{\subsection}{\color{primary}\sffamily\normalsize\bfseries}{\thesubsection}{0.7em}{}
\titlespacing*{\section}{0pt}{8pt}{4pt}
\titlespacing*{\subsection}{0pt}{5pt}{2pt}
\setlist[itemize]{left=1.2em, itemsep=2pt, topsep=2pt, parsep=0pt}
\setlist[enumerate]{left=1.4em, itemsep=2pt, topsep=2pt, parsep=0pt}

\tcbset{abstractstyle/.style={enhanced, colback=white, colframe=primary, boxrule=1pt,
  fonttitle=\bfseries\sffamily\large, left=6mm, right=6mm, top=2mm, bottom=2mm, width=\textwidth, boxsep=4pt, breakable}}
\makeatletter
\def\@maketitle{%
  \begin{center}
    {\fontsize{18pt}{21pt}\selectfont \bfseries \textcolor{primary}{Sentence-Level Context Sensitivity as a Training-Free Detector\\[2pt]
      of Unsupported Content, Evaluated Against Trained Verifiers}}\\[1ex]
    {\normalsize
      Mohamed Aly Bouke\,%
      \raisebox{0.6ex}{\href{https://orcid.org/0000-0003-3264-601X}{\includegraphics[height=1.4ex]{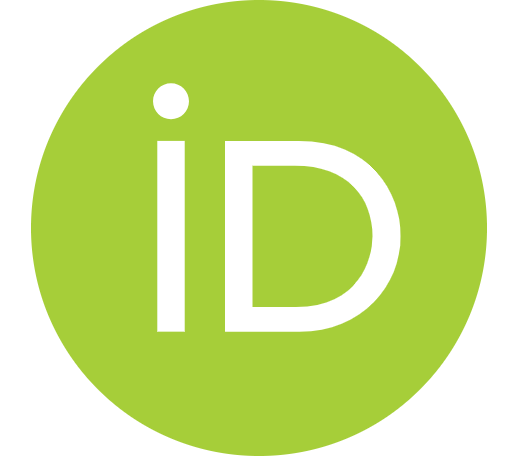}}}%
      \,\raisebox{0.6ex}{\href{mailto:bouke@ieee.org}{\textcolor{gray}{\scriptsize\faEnvelope}}}%
      \raisebox{1.3ex}{\scriptsize1,*}%
    }\\[0.8ex]
    {\footnotesize
      \textsuperscript{1}Centre for Intelligent Cloud Computing, CoE for Advanced Cloud, Faculty of Information Science and Technology,\par
      Multimedia University, Jalan Ayer Keroh Lama, Bukit Beruang, 75450, Melaka, Malaysia
    }\\[0.5ex]
    {\scriptsize \texttt{*bouke@ieee.org, alybouke@mmu.edu.my}}\\[1ex]
    {\scriptsize \textit{Research Article}, \today}
  \end{center}
}
\renewcommand{\maketitle}{\twocolumn[\color{textgray}\@maketitle\vspace{-1.0em}]}
\makeatother

\begin{document}
\maketitle

\begin{strip}
  \begin{center}
    \begin{tcolorbox}[abstractstyle, title=Abstract]
      \normalsize
      Retrieval-augmented generation (RAG) assistants summarize records in clinical and legal work, where one unsupported sentence can mislead a reader. The contrast between an
      output's likelihood with and without its source is an established faithfulness score for whole
      summaries and answers, but it has not been measured as a detector of the individual unsupported
      sentence in multi-passage RAG answers, against trained verifiers, or for its cost. We implement it as
      a training-free detector that re-scores a fixed
      answer under the full context, no context, and each chunk removed, and returns the chunk whose removal
      lowers a sentence's likelihood most as a candidate supporting passage. We evaluate it on RAGTruth,
      TofuEval, and RAGBench with six scorers and against five verifiers, up to a large language model (LLM) judge, on identical inputs under a source-level split. Scoring per sentence ranks
      unsupported sentences better than the answer-level form of the same signal on all three benchmarks,
      by 0.033 to 0.071 in the area under the receiver operating characteristic curve (AUC). On RAGTruth the
      training-free score reaches an AUC of $0.717$ to $0.745$ across scorers and $0.773$ with a classifier,
      above entailment and attribution baselines and level with per-chunk fact-checkers, at about one
      forty-seventh of the LLM judge's compute on a $1.5$B scorer, while a full-context fact-checker and the
      judge are more accurate and are not improved by it. The signal is weakest on short-answer question answering, where the scorer can answer from memory.
    \end{tcolorbox}
  \end{center}
  \vspace{3pt}
  {\noindent\small\textbf{Keywords:}\, hallucination detection, retrieval-augmented generation, context perturbation, explainable AI, large language models\par}
\end{strip}

\section{Introduction}\label{sec:intro}
LLM assistants increasingly answer from retrieved evidence through RAG~\cite{lewis2020rag}, and they now summarize reports, meetings, and records in settings such as clinical decision support and legal research. Grounding the answer in retrieved text narrows the gap between what a model asserts and what can be checked, yet it does not remove hallucination, the production of fluent statements that the supplied evidence does not support~\cite{ji2023survey,huang2025survey}. A single unsupported sentence in an otherwise faithful summary can mislead a reader who has no reason to doubt it, so a useful check must say which sentence is unsupported and show the evidence a reviewer should inspect.

Existing detectors meet this requirement only in part. Intrinsic signals such as perplexity or token uncertainty describe how confidently a model writes, and an unsupported sentence can be as fluent as a supported one. Entailment models and trained fact-checkers test support directly and rank unsupported sentences well, but they require a separate model, trained on labeled data and matched to the domain, and a prompted LLM judge adds a large inference cost for every sentence. What is missing is a cheap signal that needs no trained verifier, works at the level of the sentence, and returns the passage it relied on, together with an honest account of where such a signal falls short of the trained alternatives.

A candidate signal is already known. If a sentence is supported by a retrieved passage, a probabilistic scorer tends to assign it lower likelihood once that passage is removed, whereas a sentence produced mostly from the model's prior tends to change little. The contrast between an output's likelihood with and without its source has been used to score the faithfulness of summaries~\cite{xie2021coco,jia2023fflm}, of dialogue responses as conditional pointwise mutual information~\cite{nandwani2023pmifaith}, and of whole answers in open-book question answering~\cite{vankov2025consens}, and removing single passages is an established way to attribute an output to its sources~\cite{cohenwang2024contextcite,liu2025attribot}. What has not been measured is whether this contrast localizes the unsupported sentence inside a multi-passage RAG answer, how the sentence-level signal compares with its answer-level form and with trained verifiers on identical inputs, what it costs, and where it fails. We call the sentence-level reaction grounding sensitivity and treat it as an empirical dependence signal, not as a definition of semantic support, since prior knowledge, redundant evidence, and topical priming can weaken it.

The objective of this paper is to measure these properties under an evaluation that keeps every source document on one side of the train and test split, without claiming a new signal. We treat outputs constructed from the retrieved context, such as summarization and data-to-text, as the primary scope and short-answer question answering, where the answer can come from parametric knowledge, as its boundary. The contributions are as follows.
\begin{enumerate}
  \item We implement the context-sensitivity signal as a training-free, sentence-level detector that also returns a candidate supporting passage, and measure what sentence-level scoring adds over the answer-level form of the same signal.
  \item We evaluate it on three benchmarks with six scorers against five verifiers on identical inputs, and assess its attribution with an automatic judge and a three-annotator study.
  \item We characterize where it sits on the cost-accuracy frontier and which verifiers it complements, against a perplexity-and-length control.
\end{enumerate}
Section~\ref{sec:related} reviews related work, Section~\ref{sec:problem} defines the task, Section~\ref{sec:method} presents the method, Section~\ref{sec:setup} describes the evaluation, Sections~\ref{sec:results} and~\ref{sec:analysis} report results and diagnostics, Section~\ref{sec:discussion} discusses use and failure modes, Section~\ref{sec:limits} states limitations, and Section~\ref{sec:conclusion} concludes.

\section{Related Work}\label{sec:related}
Detectors of unsupported content differ in where they look for evidence, whether in the model's confidence, in the agreement of repeated samples, in a separate verifier, or in how the output changes when the context is edited. Each family answers part of the grounding question and leaves a different part of it open.

Intrinsic detectors read token probability, entropy, or uncertainty aggregated over meaning-equivalent generations~\cite{farquhar2024semantic}, and recent work reads hidden states or attention statistics for the same purpose~\cite{chen2024inside,sriramanan2024llmcheck}. They need no external resources, but fluency and confidence are weak proxies for grounding, and calibration shifts across models and prompts. Sampling-consistency methods such as SelfCheckGPT~\cite{manakul2023selfcheckgpt} treat disagreement among repeated generations as evidence of hallucination, which is informative when a model is unsure but misses an unsupported claim the model repeats consistently, and it multiplies the inference budget. Both families probe the model's own state or stochasticity, so neither tests whether an output depends on the evidence it was given.

Natural language inference (NLI) models and fact-verification pipelines test whether the context entails each claim, including fine-grained scoring over atomic facts~\cite{min2023factscore}, and trained verifiers such as AlignScore and MiniCheck~\cite{zha2023alignscore,tang2024minicheck} or detectors trained on RAG data~\cite{kovacs2025lettucedetect} are among the most accurate options available. Evaluation suites for RAG apply prompted judges in the same way~\cite{es2024ragas}. These methods target grounding directly, but each depends on a separate model whose training data, domain, and input length may not match the deployment, and the stronger ones require labeled supervision or a large judge. They leave open what a signal read from any scorer, without such a model, can achieve and where it stops being competitive.

The contrast between an output's likelihood with and without its source is itself established as a faithfulness signal. CoCo masks the source content relevant to a summary and measures the drop in token probability~\cite{xie2021coco}, Cao et al.\ compare the prior and posterior probability of summary entities~\cite{cao2022hallucinated}, FFLM combines probability changes with and without the source into a zero-shot score~\cite{jia2023fflm}, PMI-Faith scores dialogue responses by conditional pointwise mutual information with the grounding document~\cite{nandwani2023pmifaith}, and ConSens scores open-book answers by the ratio of their perplexity without and with the context, computed by a separate small evaluator~\cite{vankov2025consens}. These scores operate on a whole summary, response, or answer. Context attribution removes sources to credit a generation to the passages that produced it, through leave-one-out drops~\cite{liu2025attribot}, surrogate models over random ablations~\cite{cohenwang2024contextcite}, or divergences between output distributions with and without each context sentence~\cite{li2026arcjsd}, and establishes that removing a source changes the output. It is framed as attribution for an answer assumed worth citing, and not as a decision about whether any passage supports a span. Attention-based detection flags a span when little attention falls on the context~\cite{chuang2024lookback}, and mechanistic methods trace how feed-forward and copying components use parametric knowledge and retrieved evidence~\cite{sun2025redeep,yeh2026lumina}. These signals are informative but correlational, and the mechanistic ones require white-box access to the generator.

The signal studied here is therefore not new. The likelihood contrast is established at the level of whole outputs and the leave-one-out drop is established for attribution, but no study measures the contrast at the level of individual sentences inside multi-passage RAG answers, against its answer-level form and against trained verifiers under one protocol, with its cost and its failure modes. That measurement is the contribution of this paper, and Table~\ref{tab:compare} places the detector among the families.

\begin{table*}[t]
\centering
\caption{Qualitative comparison of hallucination-detection families along the axes that matter for grounded, explainable detection, with a representative method for each family.}
\label{tab:compare}
\footnotesize
\setlength{\tabcolsep}{4pt}
\renewcommand{\arraystretch}{1.2}
\begin{tabular}{@{}p{2.3cm}p{2.3cm}p{1.95cm}p{2.0cm}p{1.35cm}p{1.5cm}p{2.6cm}@{}}
\toprule
Family & Signal source & What is perturbed & Granularity & Separate verifier & Built-in explanation & Example methods \\
\midrule
Intrinsic / uncertainty & token probability, entropy & nothing & response or token & no & no & semantic entropy, INSIDE~\cite{farquhar2024semantic,chen2024inside} \\
\rowsep
Sampling consistency & agreement of samples & the output (resampling) & response or sentence & no & no & SelfCheckGPT~\cite{manakul2023selfcheckgpt} \\
\rowsep
Entailment / fact check & external NLI or verifier & nothing & sentence or fact & yes & partial & FActScore~\cite{min2023factscore} \\
\rowsep
Trained detector & supervised classifier & nothing & span or token & yes & partial & LettuceDetect~\cite{kovacs2025lettucedetect} \\
\rowsep
Context attribution & change under ablation & the context (sources) & source attribution & no & yes & ContextCite~\cite{cohenwang2024contextcite} \\
\rowsep
Mechanistic / internal state & attention, feed-forward activations & nothing & response or token & no & partial & Lookback Lens, ReDeEP, LUMINA~\cite{chuang2024lookback,sun2025redeep,yeh2026lumina} \\
\rowsep
\textbf{Sentence-level sensitivity} & change under ablation & the context (chunks) & span & no & yes & this work \\
\bottomrule
\end{tabular}
\end{table*}

\section{Problem Setting}\label{sec:problem}
In RAG a query $q$ retrieves a context $c$ and an autoregressive model produces an answer $y$. We target grounding, whether a statement is supported by the evidence the system was given, and not factuality in the world, because a RAG system promises to answer from its passages. The benchmark with type labels distinguishes baseless information, which adds details absent from the evidence, from evident conflict, which contradicts it~\cite{niu2024ragtruth}.

The context is partitioned into chunks $c=\{s_1,\dots,s_K\}$ and the answer into sentence spans $y=(y^1,\dots,y^m)$ with character offsets. A span receives the label $Y^j=1$ when it overlaps a human-annotated hallucination, and the response label is $Y=\max_j Y^j$. A detector is a score $f$ with higher values meaning more suspicion, and its quality is the AUC, the probability that a random hallucinated span outranks a random grounded one,
\begin{equation}
\mathrm{AUC}(f) \;=\; \Pr\big[f(u^{+}) > f(u^{-})\big].
\label{eq:auc}
\end{equation}
The AUC is threshold-free and insensitive to the class prior, which suits a balanced evaluation set used for measurement, while a deployment still chooses an operating point for its review budget. Spans of one answer, and answers over one source document, share wording and evidence, so a classifier trained on some of them can recognize the rest. Every split therefore assigns each source document wholly to one side.

The single-context intrinsic baseline is perplexity,
\begin{equation}
\mathrm{PPL}(y\mid c) \;=\; \exp\!\Big(-\tfrac{1}{|y|}\sum_{t} \log p(y_t\mid y_{<t},c)\Big),
\label{eq:ppl}
\end{equation}
which cannot tell whether a probability came from the evidence or from the prior. Comparing predictive distributions across contexts uses the Jensen-Shannon divergence (JSD),
\begin{equation}
\mathrm{JSD}(P\,\|\,Q) \;=\; \tfrac{1}{2}\,\mathrm{KL}(P\,\|\,M) \;+\; \tfrac{1}{2}\,\mathrm{KL}(Q\,\|\,M),
\quad M=\tfrac{1}{2}(P+Q),
\label{eq:jsd}
\end{equation}
which is symmetric, finite when supports differ, and bounded by $\ln 2$, unlike the Kullback-Leibler divergence $\mathrm{KL}$. Table~\ref{tab:notation} in Appendix~\ref{app:details} summarizes the notation.

\section{Method}\label{sec:method}
The detector re-scores a fixed answer under edited contexts and reads how much each sentence depends on the evidence. This section defines the signal, builds the features, and states how they yield a decision and a candidate supporting passage at a known cost.

\begin{figure*}[t]
\centering
\begin{tikzpicture}[>=Latex, font=\footnotesize, every node/.style={align=center},
  io/.style={draw=primary, thick, rounded corners, fill=primary!7, inner sep=4pt, text width=6.8cm, minimum height=9mm},
  proc/.style={draw=primary, thick, fill=white, inner sep=4pt, minimum height=9mm},
  cond/.style={draw=primary, thick, fill=primary!4, inner sep=3pt, text width=3.9cm, minimum height=12mm},
  dec/.style={draw=accent, thick, fill=accent!12, diamond, aspect=2.4, inner sep=1pt, text width=3.2cm},
  outcome/.style={draw=primary, thick, rounded corners, inner sep=4pt, text width=4.9cm, minimum height=13mm},
  ar/.style={->, primary, thick, rounded corners}]
\node[io] (in) {\textbf{Inputs:} the question, the retrieved sources, and the AI answer to check};
\node[proc, below=4mm of in, text width=5.4cm] (chunk) {Split the sources into pieces};
\node[cond, below=11mm of chunk, minimum height=17mm, text width=4.3cm] (full) {Re-read with \textbf{all} sources, and record the token likelihoods and predictive distributions};
\node[cond, left=8mm of full, minimum height=17mm, text width=4.3cm] (noc) {Re-read with \textbf{no} sources, and record the token likelihoods and predictive distributions};
\node[cond, right=8mm of full, minimum height=17mm, text width=4.3cm] (loo) {Re-read with \textbf{one piece removed} at a time, and record the token likelihoods and predictive distributions};
\node[proc, below=12mm of full, text width=9cm] (calc) {Measure likelihood drops and distribution shifts from the all-sources reading when sources are removed};
\node[proc, below=4mm of calc, text width=7cm] (agg) {Combine the drops for each sentence};
\node[proc, below=4mm of agg, text width=7cm] (clf) {Turn the combined drops into a grounding-sensitivity score for each sentence};
\node[dec, below=5mm of clf, text width=3.6cm] (dec) {Does the sentence show high grounding sensitivity?};
\node[outcome, below left=8mm and 9mm of dec, fill=primary!7] (g) {\textbf{Predicted grounded}\\ show the candidate supporting passage};
\node[outcome, below right=8mm and 9mm of dec, fill=accent!12] (h) {\textbf{Predicted unsupported}\\ flag for review};
\draw[ar] (in)--(chunk);
\coordinate (d) at ($(chunk.south)+(0,-6mm)$);
\draw (chunk.south)--(d);
\draw (noc.north|-d)--(loo.north|-d);
\draw[ar] (noc.north|-d)--(noc.north);
\draw[ar] (d)--(full.north);
\draw[ar] (loo.north|-d)--(loo.north);
\coordinate (m) at ($(full.south)+(0,-7mm)$);
\draw (noc.south)--(noc.south|-m);
\draw (loo.south)--(loo.south|-m);
\draw (full.south)--(m);
\draw (noc.south|-m)--(loo.south|-m);
\draw[ar] (m)--(calc.north);
\fill[primary] (d) circle (1.5pt);
\fill[primary] (m) circle (1.5pt);
\draw[ar] (calc)--(agg);
\draw[ar] (agg)--(clf);
\draw[ar] (clf)--(dec);
\draw[ar] (dec.west) -| node[near start, above, font=\scriptsize]{yes} (g.north);
\draw[ar] (dec.east) -| node[near start, above, font=\scriptsize]{no} (h.north);
\end{tikzpicture}
\caption{Overview of the detector. A fixed answer is re-scored under the full context, under no context, and under each
retrieved chunk removed in turn. The resulting per-sentence grounding-sensitivity score predicts each sentence as
grounded, with a candidate supporting source, or flags it for review; the outcomes are the detector's predictions, not
verified verdicts.}
\label{fig:pipeline}
\end{figure*}
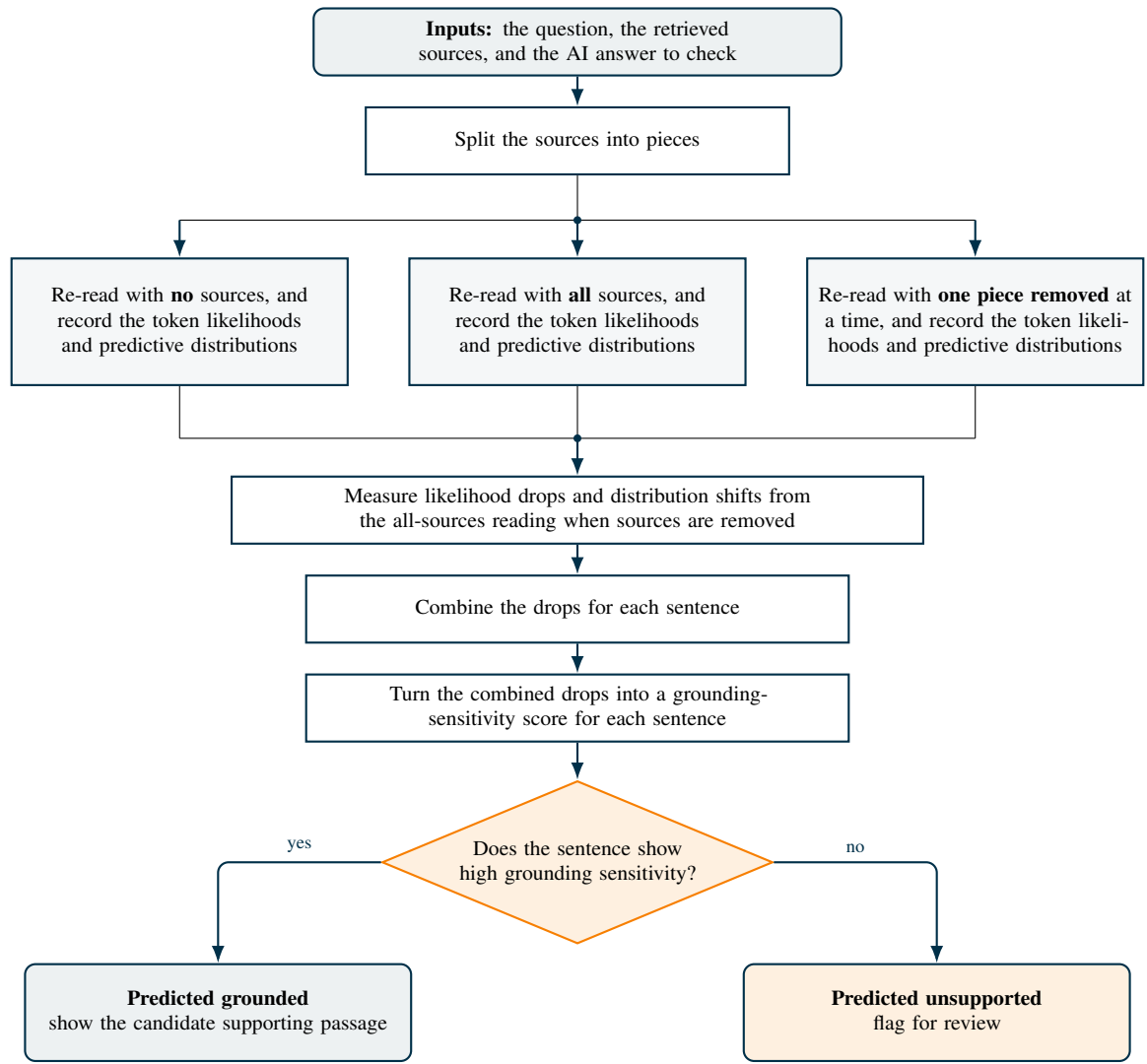

\subsection{Grounding sensitivity}\label{sub:m-def}
\begin{definition}[Grounding sensitivity]\label{def:gs}
For an answer span $y$ and a context unit $s \subseteq c$, the grounding sensitivity of $y$ with respect to $s$ is the change in the scorer's readouts of $y$ when $s$ is removed from the context, instantiated as the reduction in the mean token log-likelihood of $y$ and the JSD between its per-token predictive distributions before and after the removal.
\end{definition}
We expect, and test empirically, that removing a passage a span draws on lowers its likelihood and shifts its predictive distribution, while a span that depends on no passage barely changes. Two formal properties hold and are stated in Section~\ref{sec:formal}. A span whose predictive distribution is unchanged by every removal has zero sensitivity, and the divergence features lie in $[0,\ln 2]$. They fix the limiting case and the scale of the features and say nothing about detection accuracy, which is established only empirically. Removing a chunk also shortens the context and moves later tokens, so part of the measured change can reflect such confounds, which is why the signal is used as a continuous score and not as a test of exact independence.

\subsection{Features and aggregation}\label{sub:m-feat}
The fixed answer is scored under the full context $c$, the empty context $\varnothing$, and each context with one chunk removed, $c\setminus s_k$ for $k=1,\dots,K$. Dependence on the whole context is the mean log-likelihood gap, whose average over a whole answer is the answer-level score of ConSens~\cite{vankov2025consens} and a pointwise mutual information~\cite{nandwani2023pmifaith}, and its distributional counterpart,
\begin{equation}
\textsf{gap}(y) = \frac{1}{|y|}\sum_{t\in y}\Big[\log p(y_t\mid y_{<t},c) - \log p(y_t\mid y_{<t},\varnothing)\Big],
\label{eq:gap}
\end{equation}
\begin{equation}
\textsf{jsd}_{\varnothing}(y) = \frac{1}{|y|}\sum_{t\in y}\mathrm{JSD}\!\big(p^{c}_{t}\,\|\,p^{\varnothing}_{t}\big),
\label{eq:jsdn}
\end{equation}
and dependence on the single most supportive chunk is the largest leave-one-out drop and its counterpart,
\begin{equation}
\begin{split}
\textsf{drop}(y) = \max_{k}\frac{1}{|y|}\sum_{t\in y}\Big[&\log p(y_t\mid y_{<t},c)\\
&{}- \log p(y_t\mid y_{<t},c\setminus s_k)\Big],
\end{split}
\label{eq:drop}
\end{equation}
\begin{equation}
\textsf{jsd}_{\mathrm{loo}}(y) = \max_{k}\frac{1}{|y|}\sum_{t\in y}\mathrm{JSD}\!\big(p^{c}_{t}\,\|\,p^{c\setminus s_k}_{t}\big).
\label{eq:jsdloo}
\end{equation}
Because a claim can rest on several passages, the leave-one-out drop is also summarized by its mean and its top-two average, which gives six features. All conditions share one answer tokenization, so each token is assigned to the sentence containing its starting character offset and the per-token values are averaged over the sentence. The same offsets align sentences with the character-level annotations, and span-level scoring adds no model calls beyond the response-level passes. Table~\ref{tab:featdef} in Appendix~\ref{app:details} summarizes the features.

\subsection{Detection, attribution, and cost}\label{sub:m-detect}
The features are larger for grounded spans, so the alarm is low sensitivity. The default detector, the two-pass gap, is training-free and uses the gap alone, with its direction fixed a priori,
\begin{equation}\label{eq:hgap}
  H_{\mathrm{gap}}(y) = -\,\mathrm{gap}(y),
\end{equation}
flagging a span when $H_{\mathrm{gap}}$ exceeds an operating threshold, the only quantity chosen on labeled development data. For comparison, a trained variant fits a class-balanced logistic regression on the six features, and a combined variant adds perplexity and length to that classifier, reported as Sensitivity (trained) and Sensitivity + baselines. The two-pass gap needs two scoring passes, whereas the $K$ leave-one-out passes serve the trained variants and the attribution. The attribution returns, for each sentence, the chunk that attains the maximum in Eq.~\eqref{eq:drop} as a candidate supporting passage. It comes from the same perturbations that measure dependence, not from a separate explanation model, but it is a candidate for inspection and not a verified citation, and Section~\ref{sub:r-explain} assesses it.

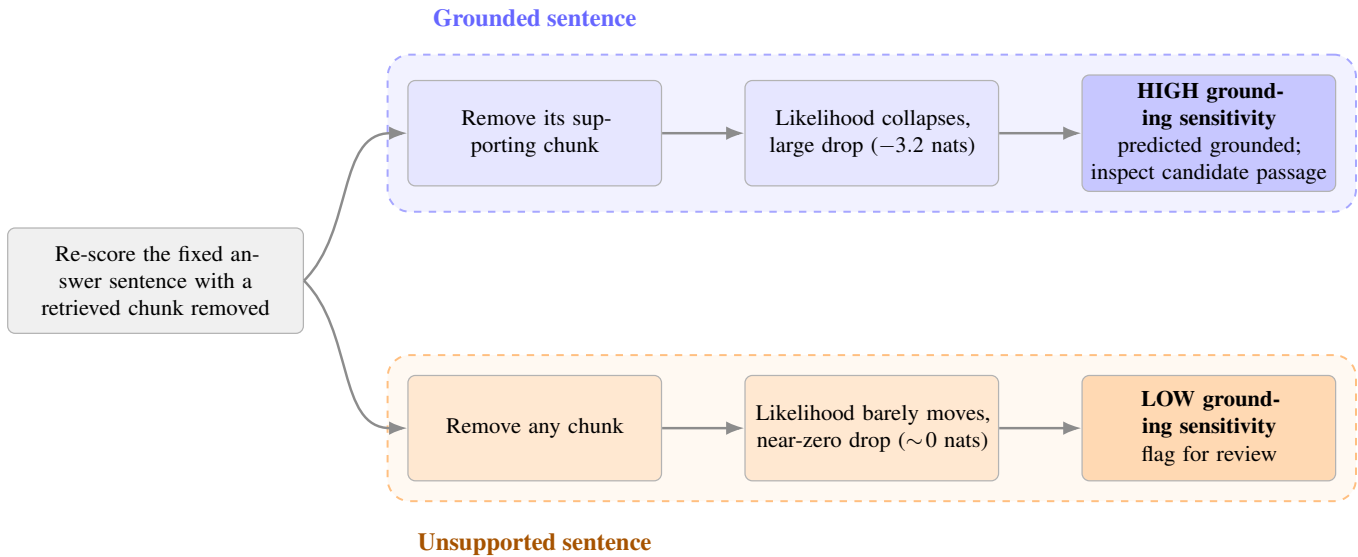
\begin{figure*}[t]
\centering
\resizebox{\textwidth}{!}{%
\begin{tikzpicture}[font=\sffamily,
  >={Latex[length=2.6mm]},
  ar/.style={->,line width=1pt,black!45,rounded corners=2pt},
  ib/.style={rounded corners=3pt,draw=black!30,line width=0.5pt,align=center,
    text width=34mm,minimum height=15mm,inner sep=3pt,font=\small},
  gbox/.style={ib,fill=blue!10},
  abox/.style={ib,fill=orange!18},
  gleaf/.style={ib,fill=blue!22,font=\small\bfseries},
  aleaf/.style={ib,fill=orange!30,font=\small\bfseries},
  root/.style={ib,fill=black!6,text width=40mm}]

  \node[root](r) at (0,0) {Re-score the fixed answer sentence with a retrieved chunk removed};

  \node[gbox](g1) at (5.4,2.1) {Remove its supporting chunk};
  \node[gbox](g2) at (10.2,2.1) {Likelihood collapses,\\large drop ($-3.2$ nats)};
  \node[gleaf](g3) at (15.0,2.1) {HIGH grounding sensitivity\\{\mdseries predicted grounded;\\inspect candidate passage}};

  \node[abox](h1) at (5.4,-2.1) {Remove any chunk};
  \node[abox](h2) at (10.2,-2.1) {Likelihood barely moves,\\near-zero drop ($\sim\!0$ nats)};
  \node[aleaf](h3) at (15.0,-2.1) {LOW grounding sensitivity\\{\mdseries flag for review}};

  \begin{scope}[on background layer]
    \node[rounded corners=9pt,draw=blue!35,dashed,line width=0.8pt,fill=blue!5,
      fit=(g1)(g3),inner sep=8pt]{};
    \node[rounded corners=9pt,draw=orange!50,dashed,line width=0.8pt,fill=orange!5,
      fit=(h1)(h3),inner sep=8pt]{};
  \end{scope}

  \draw[ar](r.east) to[out=45,in=180] (g1.west);
  \draw[ar](r.east) to[out=-45,in=180] (h1.west);
  \draw[ar](g1.east)--(g2.west);
  \draw[ar](g2.east)--(g3.west);
  \draw[ar](h1.east)--(h2.west);
  \draw[ar](h2.east)--(h3.west);

  \node[font=\bfseries,text=blue!60] at (5.4,3.75) {Grounded sentence};
  \node[font=\bfseries,text=orange!65!black] at (5.4,-3.75) {Unsupported sentence};
\end{tikzpicture}}
\caption{The grounding-sensitivity signal, shown on two example spans and not as a general guarantee. Re-scoring a
fixed answer sentence with a retrieved chunk removed tends to separate the two cases: a grounded sentence often loses
much of its likelihood, a large drop that reads as high sensitivity and a predicted-grounded outcome, whereas an
unsupported sentence often changes little, a near-zero drop that reads as low sensitivity and is flagged for review. The
outcomes are the detector's predictions, and the drop magnitudes are representative values from a grounded and an
unsupported RAGTruth span under the $14$B scorer.}
\label{fig:concept}
\end{figure*}

\label{sub:m-algo}For $K$ chunks the method runs $K+2$ scoring passes over a fixed answer, or two for the two-pass gap, with divergence terms costing $O(AV)$ per condition for answer length $A$ and vocabulary size $V$, and it involves no generation and no training for feature construction (Algorithm~\ref{alg:gs} and Table~\ref{tab:cost} in Appendix~\ref{app:details}). Removing chunks instead of single sentences bounds the number of passes at the price of coarser localization. The maximum over chunks suits single-passage support and ignores the many removals that leave a span unchanged, while the no-context condition, the limiting case of removing every chunk, captures diffuse support, and the two together bracket concentrated and diffuse reliance.

\section{Formal Properties}\label{sec:formal}
\begin{proposition}[Insensitivity of free-running spans]\label{prop:free}
If the predictive distribution at every scored token of $y$ is unchanged by every intervention the features read, that is $p^{c}_{t}=p^{c\setminus s_k}_{t}$ for every unit $s_k$ and $p^{c}_{t}=p^{\varnothing}_{t}$ for all $t$, then all grounding-sensitivity features of $y$ are zero. Conversely, a nonzero feature implies that the predictive distribution over $y$ responds to at least one intervention.
\end{proposition}
\begin{proof}[Proof sketch]
Equal distributions give equal realized-token probabilities, so each log-likelihood difference in Eqs.~\eqref{eq:gap} and~\eqref{eq:drop} is zero, and each JSD term in Eqs.~\eqref{eq:jsdn} and~\eqref{eq:jsdloo} is the divergence of a distribution from itself. Means and maxima of zeros are zero, and the converse is the contrapositive. The hypothesis is stated on full distributions and on both kinds of intervention because equality of the realized token alone would not make the divergence or no-context terms vanish.
\end{proof}
\begin{proposition}[Boundedness of the divergence features]\label{prop:bound}
With the natural logarithm, every per-token JSD term lies in $[0,\ln 2]$, so the aggregated divergence features lie in $[0,\ln 2]$.
\end{proposition}
\begin{proof}[Proof sketch]
The JSD equals the mutual information between a fair binary mixture indicator and a sample from the mixture, which is at most one bit, and a mean or maximum of values in $[0,\ln 2]$ stays in that interval.
\end{proof}
A hallucinated span is rarely perfectly independent of the context, because topic and style leak through, so its sensitivity is small but not zero, and a nonzero feature shows only a reaction to some removal, which can include the confounds of shorter context, shifted positions, and broken formatting. The propositions therefore justify the direction and the scale of the features and nothing about their accuracy.

\section{Experimental Setup}\label{sec:setup}
The evaluation is built so that every method scores the same sentences and chunks under the same source-level split, and differences in AUC therefore reflect the methods and not their inputs.

\subsection{Benchmarks}\label{sub:s-data}
RAGTruth~\cite{niu2024ragtruth} is the primary benchmark because it is the only one of the three with human span-level annotations and a hallucination type, which the per-type, per-task, baseline, attribution, and complementarity analyses need. We sample $300$ hallucinated and $300$ clean responses across its summarization, data-to-text, and question-answering tasks. TofuEval~\cite{tang2024tofueval}, in its MeetingBank portion~\cite{hu2023meetingbank}, tests transfer to a second summarization domain with sentence-level consistency labels, mapped to the same format with the transcript as context and the topic as query ($884$ summaries, $2{,}415$ sentences, $50$ meetings). RAGBench~\cite{friel2024ragbench} probes the boundary of the scope, short-answer question answering pooled over six domains ($598$ responses, $2{,}951$ sentences), at its available class distribution. Sentence and chunk boundaries are character spans defined once over the original text, so every scorer and every method sees the same segmentation (Table~\ref{tab:data}).

\begin{table}[t]
\centering
\caption{Composition of the three evaluation benchmarks, from the canonical case set that produced the reported
results ($14$B scorer). RAGTruth covers summarization, data-to-text, and question answering, TofuEval meeting summarization, and RAGBench short-answer question answering. RAGTruth is class-balanced at the response level by sampling; TofuEval and RAGBench
follow their available distribution. Sources are the units of the leakage-clean split. Sentence counts use the Qwen
tokenizer and vary slightly across scorers.}
\label{tab:data}
\small
\begin{tabular}{lrrr}
\toprule
Quantity & RAGTruth & TofuEval & RAGBench \\
\midrule
Responses          & 600 & 884 & 598 \\
\quad hallucinated & 300 & 309 & 267 \\
\quad clean        & 300 & 575 & 331 \\
\midrule
Sentences          & 4{,}191 & 2{,}415 & 2{,}951 \\
\quad hallucinated & 534 & 446 & 776 \\
\quad grounded     & 3{,}657 & 1{,}969 & 2{,}175 \\
\midrule
Sources            & 600 & 50 & 524 \\
Scorers            & 6 & 1 & 1 \\
\bottomrule
\end{tabular}
\end{table}

\subsection{Scorers and baselines}\label{sub:s-models}\label{sub:s-base}
Six instruction-tuned scorers span four families and a ten-fold range in scale: Qwen2.5-Instruct at $1.5$B, $7$B, and $14$B~\cite{qwen2025}, SmolLM2-1.7B-Instruct~\cite{allal2025smollm2}, Llama-3.1-8B-Instruct~\cite{grattafiori2024llama3}, and Phi-3.5-mini-instruct~\cite{abdin2024phi3}. The $14$B scorer is the headline configuration, and none of the scorers generated the benchmark answers. The simple baselines are perplexity and answer length. The five verifiers are chunk-level NLI with the \texttt{cross-encoder/nli-deberta-v3-large} model, applied to the best of the $K=5$ chunks and to the whole context; AlignScore and MiniCheck, each applied per chunk and to the full context; ContextCite, which fits a linear surrogate over random ablations and reads each sentence's largest positive attribution; and an LLM judge that prompts a $14$B instruction-tuned model per sentence and reads a calibrated probability of support from its next-token logits. Mechanistic detectors~\cite{sun2025redeep,yeh2026lumina} are not run because they need white-box access to the generator, and LettuceDetect~\cite{kovacs2025lettucedetect} is not run because its released weights were trained on RAGTruth and would need retraining per split.

\subsection{Protocol and implementation}\label{sub:s-proto}
Each source document, a RAGTruth source, a TofuEval meeting, or a RAGBench evidence set, is assigned wholly to a development or a test partition in a $60/40$ split. The test partitions hold $240$ sources, $1{,}591$ sentences, and $213$ positive spans on RAGTruth, $20$ sources, $973$ sentences, and $187$ positives on TofuEval, and $210$ sources, $1{,}194$ sentences, and $309$ positives on RAGBench. Trained detectors are fit on the development partition, and the threshold of the two-pass gap is chosen there. Confidence intervals (CI) and significance tests use a source-level paired bootstrap with $2000$ resamples, and a difference is significant when its two-sided $95\%$ interval excludes zero. Contexts are truncated to $1800$ tokens and answers to $256$, the context is split into $K=5$ chunks by sentence grouping, divergences use the full vocabulary, and $K$ and the maximum aggregation were fixed before any experiment. The procedure is deterministic given the scorer and the seed, and Table~\ref{tab:hyper} in Appendix~\ref{app:details} lists the full configuration.

\section{Results}\label{sec:results}
Detection is reported at the span level throughout, since localizing the unsupported sentence is the objective. RAGTruth establishes the result across scorers, error types, and tasks before the other two benchmarks test how far it transfers and the verifiers show where it stands.

\subsection{Detection on RAGTruth}\label{sub:r-span}\label{sub:r-robust}
Table~\ref{tab:sent} reports span-level AUC on RAGTruth for the six scorers. The training-free two-pass gap reaches $0.717$ to $0.745$ and Sensitivity + baselines $0.741$ to $0.773$, against $0.597$ to $0.666$ for perplexity and about $0.55$ for length. On the $14$B scorer Sensitivity + baselines exceeds the combined perplexity-and-length baseline by $+0.117$, with a $95\%$ interval of $[0.082, 0.153]$. The spread across scorers is about three points over a ten-fold range in scale and four families, and the margin over perplexity is positive on every scorer, from $+0.080$ on Phi-3.5 to $+0.147$ on SmolLM2, so most of the signal is available from a small instruction-tuned scorer. What the method needs from a scorer is competence at using the context, since a scorer that ignores the context would react little to its removal.

\begin{table*}[t]
\centering
\caption{Span-level (sentence) AUC on RAGTruth under the leakage-clean, source-level protocol, for six scorers
across four model families. Gap (two-pass) is the training-free single-feature variant, with its
direction fixed a priori, Sensitivity (trained) the classifier on the grounding features, and Sensitivity + baselines the combined set. Best
per row in bold. The signal is stable across scale and family and consistently above perplexity.}
\label{tab:sent}
\small
\begin{tabular}{lccccc}
\toprule
Scorer & Perplexity & Length & Gap (two-pass) & Sensitivity (trained) & Sensitivity + baselines \\
\midrule
Qwen2.5-1.5B        & 0.602 & 0.548 & 0.717 & 0.724 & \textbf{0.741} \\
\rowsep
Qwen2.5-7B          & 0.663 & 0.548 & 0.733 & 0.734 & \textbf{0.746} \\
\rowsep
Qwen2.5-14B         & 0.657 & 0.548 & 0.745 & 0.763 & \textbf{0.773} \\
\rowsep
SmolLM2-1.7B        & 0.597 & 0.561 & 0.717 & 0.723 & \textbf{0.744} \\
\rowsep
Llama-3.1-8B        & 0.626 & 0.550 & 0.721 & 0.728 & \textbf{0.744} \\
\rowsep
Phi-3.5-mini        & 0.666 & 0.554 & 0.722 & 0.732 & \textbf{0.746} \\
\bottomrule
\end{tabular}
\end{table*}

\subsection{Sentence versus answer level}\label{sub:r-answer}
The same gap can be computed per answer, as ConSens does, or per sentence. Table~\ref{tab:answer} scores each test sentence by its own gap and by the gap of its whole answer, the token-weighted mean over its sentences, on the same source-level split. Sentence-level scoring ranks unsupported sentences better on all three benchmarks, by $0.057$ on RAGTruth, $0.033$ on TofuEval, and $0.071$ on RAGBench, and on RAGTruth the advantage holds for every scorer, from $+0.029$ to $+0.057$, with every interval excluding zero. Localization also helps to decide whether an answer contains an unsupported sentence at all. Scoring each answer by its lowest sentence gap instead of its mean gap raises the response-level AUC from $0.736$ to $0.814$ on RAGTruth and from $0.608$ to $0.697$ on TofuEval with the $14$B scorer, while the difference is not significant on RAGBench or for the smaller RAGTruth scorers. The answer-level scores here are computed over the scored sentences of each answer and without the content-word filter of ConSens, so they are an answer-level form of the same signal and not a reproduction of ConSens.

\begin{table}[!htb]
\centering
\caption{Sentence-level against answer-level scoring with the same gap, $14$B scorer, source-level test split. Span level ranks sentences; response level ranks answers by their mean gap or by their lowest sentence gap. $\Delta$ is the paired difference with its source-level $95\%$ interval; a dagger marks an interval that excludes zero.}
\label{tab:answer}
\scriptsize
\setlength{\tabcolsep}{2.5pt}
\begin{tabular}{lccc}
\toprule
Benchmark & Sentence gap & Answer gap & $\Delta$ [95\% CI] \\
\midrule
\multicolumn{4}{l}{\textit{Span level}} \\
RAGTruth & 0.745 & 0.690 & $+0.057$ $[+0.027,+0.086]^\dagger$ \\
\rowsep
TofuEval & 0.700 & 0.668 & $+0.033$ $[+0.004,+0.066]^\dagger$ \\
\rowsep
RAGBench & 0.603 & 0.535 & $+0.071$ $[+0.027,+0.127]^\dagger$ \\
\midrule
\multicolumn{4}{l}{\textit{Response level (lowest sentence gap vs answer gap)}} \\
RAGTruth & 0.814 & 0.736 & $+0.077$ $[+0.033,+0.119]^\dagger$ \\
\rowsep
TofuEval & 0.697 & 0.608 & $+0.089$ $[+0.037,+0.133]^\dagger$ \\
\rowsep
RAGBench & 0.643 & 0.604 & $+0.040$ $[-0.041,+0.124]$ \\
\bottomrule
\end{tabular}
\end{table}

\label{sub:r-type}\label{sub:r-task}On the $14$B scorer Sensitivity + baselines detects baseless-information spans at $0.770$ and evident-conflict spans at $0.777$, so the same likelihood reaction captures both unsupported additions and contradictions. By task, question answering has the highest point estimate at $0.800$, summarization $0.775$, and data-to-text $0.736$ (Table~\ref{tab:pertask}), but with $37$ to $121$ positive spans per task the intervals overlap and the tasks cannot be ranked. The breakdown shows that the weakness on RAGBench reported next is not a property of question answering as a format.

\begin{table}[t]
\centering
\caption{Span-level AUC of Sensitivity + baselines by task on RAGTruth, $14$B scorer, with source-level bootstrap intervals. The
signal holds across all three tasks; the intervals overlap, so the tasks cannot be ranked with confidence, but question
answering has the highest point estimate, so the weakness on RAGBench is not a property of the question-answering
format.}
\label{tab:pertask}
\small
\begin{tabular}{lcc}
\toprule
RAGTruth task & Positives / spans & Span AUC [95\% CI] \\
\midrule
Summarization & 37 / 428 & 0.775 [0.706,\,0.837] \\
\rowsep
Data-to-text  & 121 / 661 & 0.736 [0.697,\,0.776] \\
\rowsep
Question answering & 55 / 502 & 0.800 [0.751,\,0.861] \\
\bottomrule
\end{tabular}
\end{table}

\subsection{Transfer across benchmarks}\label{sub:r-resp}\label{sub:r-tofueval}\label{sub:r-ragbench}
Sensitivity + baselines reaches $0.773$ on RAGTruth, $0.711$ on TofuEval, and $0.634$ on RAGBench (Table~\ref{tab:resp} and Figure~\ref{fig:sentauc}). The training-free detector transfers unchanged with only its threshold recalibrated, whereas the combined variant refits its classifier on each benchmark's development split, so its numbers show that the method reapplies and not that one trained model moves between corpora. On TofuEval the ordering of methods seen on RAGTruth largely holds, although its $50$ meetings widen the intervals. On RAGBench every method struggles. An answer to a factual question can often be produced from parametric knowledge, so a grounded answer reacts little to removing its passage, and RAGBench also has the longest contexts and the heaviest truncation (Appendix~\ref{app:trunc}). Its sample is dominated by the PubMedQA domain, near $0.61$, and the other domains are too small to report separately. These results place the scope of the method on outputs constructed from the context and mark short-answer question answering as its boundary.

\begin{table}[t]
\centering
\caption{Span-level AUC of Sensitivity + baselines ($14$B scorer) across the three benchmarks, under the leakage-clean,
source-level protocol. The signal transfers to summarization and is weakest on short-answer question answering.}
\label{tab:resp}
\small
\begin{tabular}{lc}
\toprule
Benchmark & Span-level AUC \\
\midrule
RAGTruth (three tasks) & 0.773 \\
\rowsep
TofuEval (meeting summarization)   & 0.711 \\
\rowsep
RAGBench (short answers)          & 0.634 \\
\bottomrule
\end{tabular}
\end{table}
\begin{figure}[t]
\centering
\includegraphics[width=\linewidth]{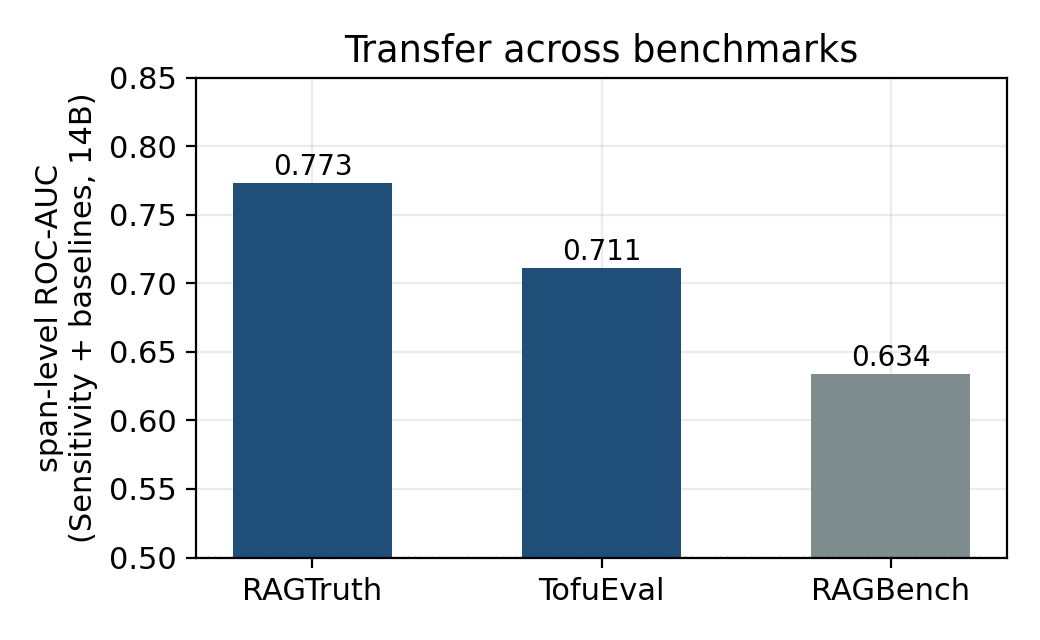}
\caption{Transfer of grounding sensitivity (Sensitivity + baselines, $14$B scorer) across the three benchmarks at the span level. The
signal is strongest on the two summarization-style benchmarks and weakest on the short-answer question answering of
RAGBench, where an answer can be recovered from parametric knowledge.}
\label{fig:sentauc}
\end{figure}

\subsection{Comparison with verifiers and cost}\label{sub:r-strongbase}
Table~\ref{tab:advbase} compares the sensitivity detectors with the five verifiers on identical inputs. On RAGTruth Sensitivity + baselines at $0.773$ is above chunk-level NLI ($0.721$) and both ContextCite variants ($0.724$ and $0.736$), with paired differences excluding zero, and is not significantly different from the per-chunk variants of AlignScore ($0.765$) and MiniCheck ($0.792$). The two verifiers that read the full context with a purpose-built or much larger model rank spans more accurately, the full-context MiniCheck at $0.831$ and the LLM judge at $0.853$. On TofuEval the combined variant at $0.711$ is above whole-context NLI and both ContextCite variants and not significantly different from chunk-level NLI, while the trained verifiers and the judge lead. On RAGBench the trained verifiers cluster with the sensitivity detectors between $0.63$ and $0.67$ without significant differences, and only the judge, at $0.714$, separates from the group.

\begin{table}[t]
\centering
\caption{Span-level AUC against five advanced verifiers on identical canonical inputs, $14$B scorer, best variant
per method. Trained verifiers (AlignScore, MiniCheck) need a separate model and labeled supervision, and the LLM judge
runs a prompted forward pass per sentence. Both sensitivity rows use only the scorer: Gap (two-pass) is the training-free
single-feature variant whose score needs no labels, and Sensitivity + baselines adds a small classifier fit on the development split.
A dagger marks a verifier that beats Sensitivity + baselines with a source-level paired interval excluding zero.}
\label{tab:advbase}
\small
\setlength{\tabcolsep}{5pt}
\begin{tabular}{lccc}
\toprule
Method & RAGTruth & TofuEval & RAGBench \\
\midrule
chunk-NLI          & 0.721 & 0.727 & 0.636 \\
\rowsep
ContextCite        & 0.736 & 0.699 & 0.668 \\
\rowsep
AlignScore         & 0.791 & 0.783$^\dagger$ & 0.650 \\
\rowsep
MiniCheck          & 0.831$^\dagger$ & 0.788$^\dagger$ & 0.647 \\
\rowsep
LLM judge ($14$B)  & 0.853$^\dagger$ & 0.881$^\dagger$ & 0.714$^\dagger$ \\
\midrule
Gap (two-pass)     & 0.745 & 0.700 & 0.603 \\
\rowsep
Sensitivity + baselines & 0.773 & 0.711 & 0.634 \\
\bottomrule
\end{tabular}
\end{table}

The value of the signal therefore lies in its cost and its requirements, not in its accuracy ceiling. Figure~\ref{fig:cost} plots accuracy against a hardware-independent compute estimate, the number of forward passes times the tokens processed times the model size. The two-pass gap on the $1.5$B scorer is the cheapest configuration, reaching $0.717$ at about $429$ GFLOPs per sentence, roughly one forty-seventh of the LLM judge ($20{,}376$), and it needs neither labels for its score nor a verifier model. Adding the leave-one-out passes and a $14$B scorer raises accuracy to $0.773$ at about $17{,}897$ GFLOPs, where the full-context MiniCheck is both more accurate and cheaper ($717$) but depends on a separately trained, supervised model. The sensitivity detectors thus occupy the training-free, low-cost end of the frontier, and the small-scorer cost should not be paired with the large-scorer accuracy.

\begin{figure}[t]
\centering
\includegraphics[width=\linewidth]{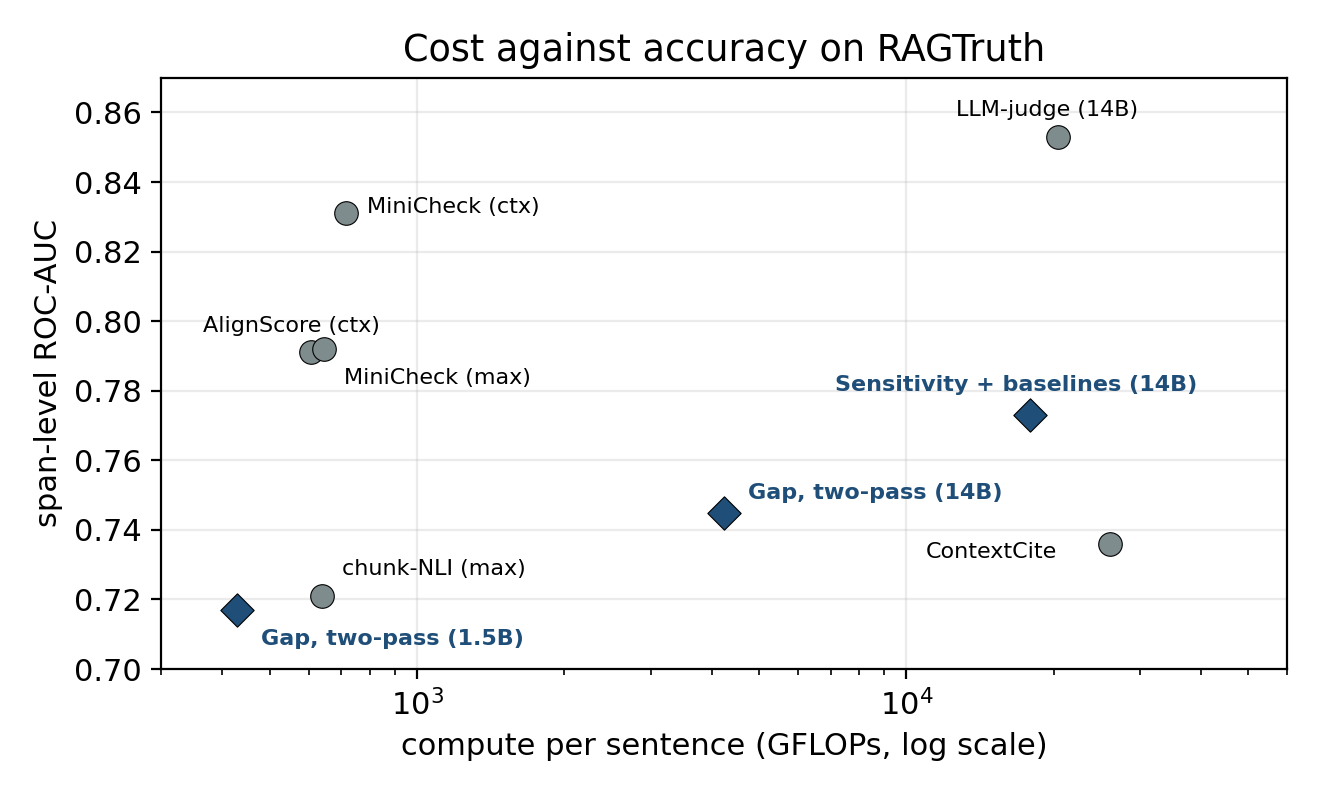}
\caption{Cost against accuracy on RAGTruth. Compute per sentence is a forward-pass estimate of floating-point operations on a log scale, with
the sensitivity detectors as navy diamonds and the advanced baselines as grey circles. The two-pass gap on a small scorer is the
cheapest point; among the methods more accurate than Sensitivity + baselines, the LLM judge is far more expensive, while the
full-context MiniCheck is cheaper than the $14$B sensitivity detector, so it trades a separate trained model for its accuracy rather
than more compute.}
\label{fig:cost}
\end{figure}

\subsection{Complementarity with verifiers}\label{sub:r-comp}
The combined sensitivity score correlates moderately with the trained verifiers, with an absolute Spearman correlation of $0.23$ to $0.46$ across benchmarks, and strongly with ContextCite ($0.66$ to $0.84$), the other method that reads a context perturbation. To test whether it adds signal, each verifier's own AUC is compared with a logistic classifier that adds the grounding features to it ($\Delta_1$), and with a control that adds perplexity and length instead ($\Delta_2$), on the same test spans (Table~\ref{tab:comp}). For chunk-level NLI, ContextCite, and the per-chunk AlignScore and MiniCheck, both differences exclude zero, and adding the grounding features raises each verifier by $0.03$ to $0.07$. For the LLM judge and the full-context MiniCheck, $\Delta_1$ spans zero; the grounding features still beat the cheap control for the judge ($\Delta_2=+0.046$) but not for the full-context MiniCheck ($\Delta_2=+0.008$, $[-0.006,+0.022]$). On RAGBench no verifier improves significantly, and on TofuEval adding them lowers the judge from $0.881$ to $0.851$. With the tested fusion, the grounding features therefore complement the weaker verifiers and do not improve the two strongest.

\begin{table*}[t]
\centering
\caption{Complementarity on RAGTruth, $14$B scorer. Each verifier's own span-ranking AUC (alone), the AUC of a
classifier adding the grounding features (\,$+$sens.\,) and one adding perplexity and length (\,$+$ppl,len\,), and two
observed paired differences with source-level $95\%$ intervals: $\Delta_1$ of $+$sens.\ over the verifier alone, and
$\Delta_2$ of $+$sens.\ over the $+$ppl,len control. Differences are computed from the unrounded AUCs, so they need not
equal the difference of the three-decimal columns; a dagger marks an interval that excludes zero. Grounding sensitivity
raises the four weaker verifiers on both differences, does not raise the two strongest on $\Delta_1$, and beats the
cheap-feature control on $\Delta_2$ for every verifier except the full-context MiniCheck.}
\label{tab:comp}
\small
\begin{tabular}{lccccc}
\toprule
Verifier & alone & $+$sens. & $+$ppl,len & $\Delta_1$ vs alone [95\% CI] & $\Delta_2$ vs $+$ppl,len [95\% CI] \\
\midrule
LLM judge            & 0.853 & 0.844 & 0.798 & $-0.009$ $[-0.024,+0.007]$ & $+0.046$ $[+0.030,+0.063]^\dagger$ \\
\rowsep
MiniCheck (full ctx) & 0.831 & 0.832 & 0.823 & $+0.000$ $[-0.012,+0.013]$ & $+0.008$ $[-0.006,+0.022]$ \\
\rowsep
MiniCheck (per chunk)& 0.792 & 0.832 & 0.792 & $+0.040$ $[+0.023,+0.058]^\dagger$ & $+0.040$ $[+0.021,+0.058]^\dagger$ \\
\rowsep
AlignScore (per chunk)& 0.765 & 0.815 & 0.770 & $+0.051$ $[+0.032,+0.070]^\dagger$ & $+0.046$ $[+0.025,+0.066]^\dagger$ \\
\rowsep
ContextCite          & 0.736 & 0.765 & 0.744 & $+0.029$ $[+0.014,+0.045]^\dagger$ & $+0.020$ $[+0.004,+0.036]^\dagger$ \\
\rowsep
chunk-NLI (per chunk)& 0.721 & 0.793 & 0.736 & $+0.073$ $[+0.048,+0.097]^\dagger$ & $+0.057$ $[+0.031,+0.083]^\dagger$ \\
\bottomrule
\end{tabular}
\end{table*}

\section{Analysis}\label{sec:analysis}
The analyses below locate the signal among the features, measure what the leave-one-out passes add, and test whether the returned passage supports the sentence, on RAGTruth with the $14$B scorer.

\subsection{Feature contributions}\label{sub:r-imp}\label{sub:r-ablation}
The two-pass gap is the strongest single feature at $0.745$, ahead of the leave-one-out drops near $0.71$ and of the two JSD features at $0.65$ and $0.61$, both below perplexity at $0.66$ (Figure~\ref{fig:imp}). Grouped by perturbation (Table~\ref{tab:ablation}), the no-context view reaches $0.753$, the leave-one-out view $0.733$, and all six features $0.763$. The comparison of the answer under the full and the empty context therefore carries most of the detection signal, which is why the training-free detector uses the gap alone, and the $K$ leave-one-out passes add a small increment to detection while providing the candidate supporting passage. Figure~\ref{fig:dist} shows the corresponding shift of the leave-one-out drop between grounded and unsupported sentences, with an overlap consistent with the moderate AUC.

\begin{figure}[t]
\centering
\includegraphics[width=\linewidth]{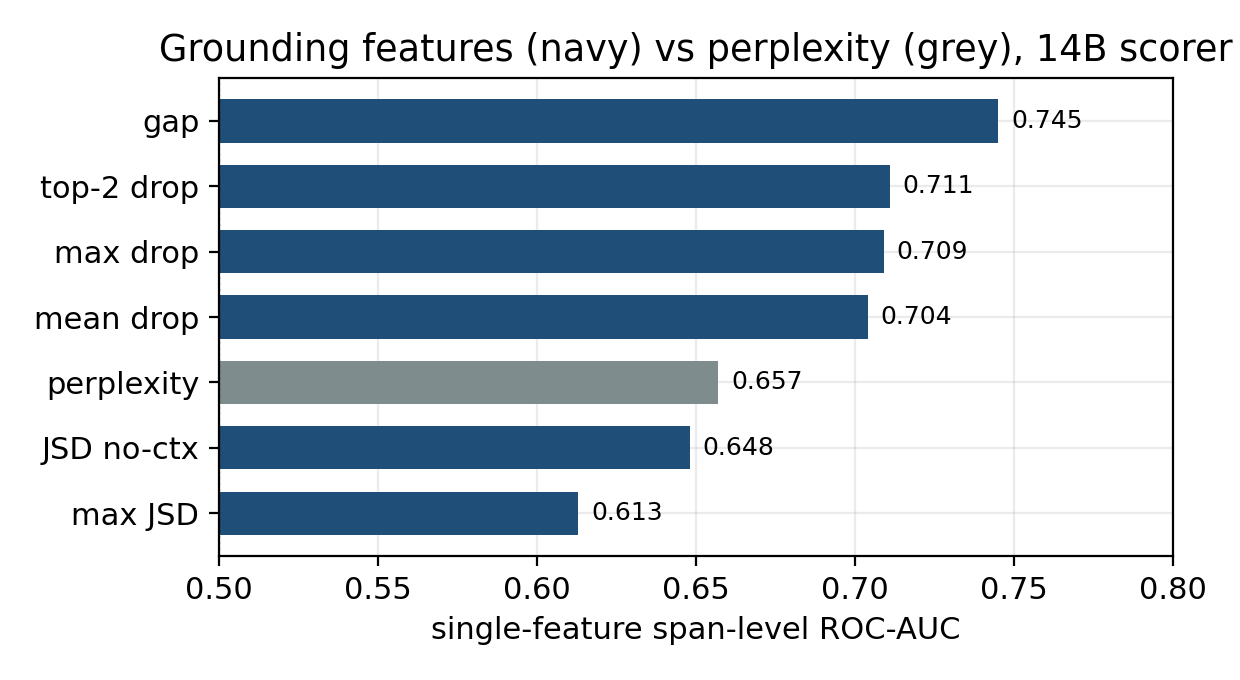}
\caption{Single-feature span-level AUC on RAGTruth, $14$B scorer, with the grounding features in navy and the
perplexity in grey. The two-pass gap is the strongest single feature, and the two divergence features fall below perplexity.}
\label{fig:imp}
\end{figure}
\begin{table}[t]
\centering
\caption{Feature-group ablation at the span level, $14$B scorer, on RAGTruth. Grouped by the perturbation each feature
reads. The two-pass no-context view carries most of the signal; the leave-one-out view adds a small increment and
supplies the attribution.}
\label{tab:ablation}
\small
\begin{tabular}{lc}
\toprule
Feature subset & Span AUC \\
\midrule
gap only (two-pass) & 0.745 \\
\rowsep
no-context view (gap, jsd$_\varnothing$) & 0.753 \\
\rowsep
leave-one-out view (drops, jsd$_\mathrm{loo}$) & 0.733 \\
\midrule
all six features & \textbf{0.763} \\
\bottomrule
\end{tabular}
\end{table}
\begin{figure}[t]
\centering
\includegraphics[width=\linewidth]{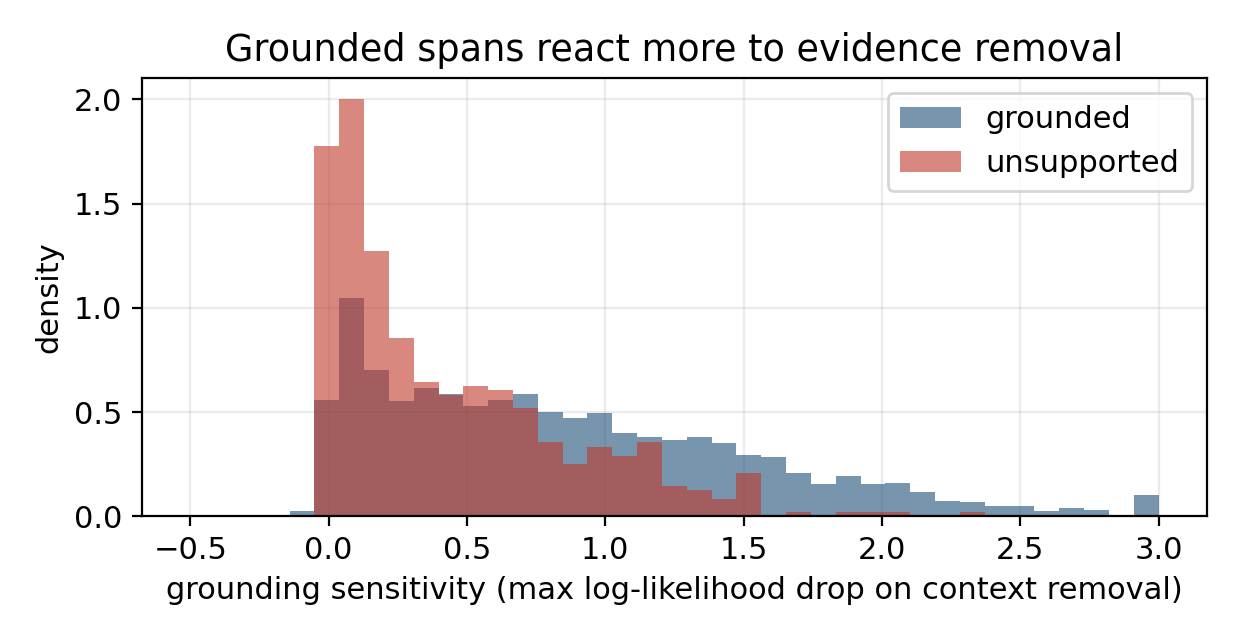}
\caption{Distribution of the maximum leave-one-out log-likelihood drop for grounded versus
unsupported sentences on RAGTruth with the $14$B scorer. Grounded spans react more to evidence removal.}
\label{fig:dist}
\end{figure}

\subsection{Attribution quality}\label{sub:r-explain}
RAGTruth marks hallucinated spans but not the passages that support grounded ones, so the returned chunk is assessed in three ways on grounded sentences. First, taking chunk-level NLI as an independent notion of support, the best-entailing chunk supports grounded sentences far more than unsupported ones, with mean entailment $0.578$ against $0.264$ (Mann-Whitney $p<10^{-40}$). Second, an instruction-tuned judge rates the returned chunk on $556$ grounded sentences as fully or partly supporting in $88.1\%$ of cases ($36.9\%$ fully and $51.3\%$ partly), against $16.4\%$ for a random chunk from the same case (paired Wilcoxon $p<10^{-70}$). Third, three annotators independently graded $120$ grounded sentences on the same four-level scale and judged the returned chunk fully supporting in $58\%$ and fully or partly supporting in $82.5\%$ of cases, with Fleiss' $\kappa$ of $0.86$ on the four levels and $0.94$ on supported versus not. The automatic and human rates come from different samples, so they do not estimate the judge's agreement with the annotators. The returned chunk is usually a supporting passage, but it is a candidate for inspection, and the study covers grounded sentences only.

\section{Discussion}\label{sec:discussion}
The results support a narrow claim and suggest a specific use. This section states what the signal measures and how it should be deployed given where it fails.

\subsection{What the evidence supports}\label{sub:d-claims}\label{sub:d-why}
The likelihood contrast is a known signal, and the evidence here concerns its sentence-level use. Grounding sensitivity measures how much the scorer's probability for a sentence depends on the retrieved evidence, a counterfactual property of the sentence in its context and not an intrinsic property of the text. Under a source-level protocol with paired tests it ranks unsupported sentences above grounded ones better than perplexity, length, whole-context entailment, and context attribution, matches the per-chunk trained verifiers, and falls short of the full-context fact-checker and the LLM judge, and scoring it per sentence ranks unsupported sentences better than scoring it per answer on every benchmark. The formal properties of Section~\ref{sec:formal} fix only the limiting case and the scale of the features, the dependence it measures need not coincide with the generator's actual use of a passage, and the evidence does not establish that dependence equals semantic support. The signal is strongest where the answer is constructed from the context and weakest where a model can answer from memory, which is the scope stated in Section~\ref{sec:intro}. Because baseless additions and contradictions are detected about equally by the magnitude of the reaction, a signed variant that separates a reliance, which lowers likelihood on removal, from a contradiction, which can raise it, is a direct refinement.

\subsection{Use and failure modes}\label{sub:d-use}\label{sub:d-fail}\label{sub:d-threat}
Given its accuracy, the detector suits triage and review prioritization in a RAG pipeline after generation, highlighting low-sensitivity sentences and showing the candidate passage for the others, and it should not gate content without human review. Its operating threshold must be set from the cost of review and a representative class prior, since a live stream is dominated by grounded content. Because the scorer is independent of the generator, the check can audit a hosted model whose internals are unavailable. Two failure modes follow from the mechanism. An unsupported sentence primed by the topic of the context can still react to its removal and escape detection, and a grounded sentence whose support is split across chunks or duplicated in several can show a small drop for every single removal and be flagged. Both point to chunking that follows the units of evidence. The intended threat is an honest generator that occasionally drifts, and an adversary could phrase an unsupported sentence to borrow enough from a passage to react under removal, although tying a sentence to a passage in this way also pulls it toward being grounded in it. Robustness to such adaptive evasion is not evaluated here.

\section{Limitations and Responsible Use}\label{sec:limits}\label{sec:threats}\label{sec:ethics}
The absolute AUC is moderate, $0.71$ to $0.77$ on the summarization-style benchmarks and $0.63$ on short-answer question answering, and two verifiers rank sentences more accurately. The labels are the benchmark annotations, whose noise bounds measurable performance. The evidence comes from three English benchmarks and six scorers, so the numbers may differ in other domains, languages, and model lineages, and the class-balanced RAGTruth sample measures ranking quality, not prevalence. Truncation is negligible on RAGTruth and TofuEval but affects roughly a quarter of RAGBench cases (Appendix~\ref{app:trunc}). The comparison excludes white-box mechanistic detectors and a RAGTruth-trained judge, which operate under different assumptions, and complementarity was tested with one fusion rule. Future work includes a signed sensitivity, chunking aligned with retriever passages, cheaper approximations of the perturbation, and a cascade that escalates flagged sentences to a stronger verifier. A detector of this kind can suppress correct statements if used as a silent filter and can teach users to stop checking, so it should surface and explain its flags for a person to judge. It checks consistency with the provided sources and not truth in the world, so a sentence faithfully grounded in a wrong passage is not flagged, and showing the passage lets a reviewer judge the source as well.

\section{Conclusion}\label{sec:conclusion}
The contrast between an output's likelihood with and without its evidence is an established faithfulness signal, and this paper measures it at the level of individual sentences in multi-passage RAG answers, implemented as a training-free detector that also returns a candidate supporting passage. Scoring per sentence ranks unsupported sentences better than the answer-level form of the same signal on all three benchmarks. Under a source-level evaluation on three benchmarks, the training-free two-pass score ranks unsupported sentences above grounded ones better than intrinsic, entailment, and attribution baselines, matches per-chunk trained verifiers, and costs about one forty-seventh of an LLM judge's compute on a small scorer, while trained full-context verifiers and the judge remain more accurate and are not improved by adding it. The signal is strongest on summarization and data-to-text and weakest on short-answer question answering. Grounding sensitivity is therefore a cheap, verifier-free first check with a candidate supporting passage, suited to outputs built from the context and to complementing weaker verifiers, not a replacement for the strongest ones.

\appendix
\section{Implementation Details}\label{app:details}
Algorithm~\ref{alg:gs} gives the scoring procedure, Table~\ref{tab:cost} its cost, Table~\ref{tab:featdef} the features and baselines, Table~\ref{tab:notation} the notation, and Table~\ref{tab:hyper} the configuration used for every reported result.

\begin{algorithm}[t]
\SetAlgoLined
\KwIn{answer $y$, query $q$, context $c$ split into chunks $s_1,\dots,s_K$, model $M$}
\KwOut{per-span grounding-sensitivity features and supporting unit}
$\ell^{c}, P^{c} \leftarrow$ score $y$ with $M$ under $(q,c)$\;
$\ell^{\varnothing}, P^{\varnothing} \leftarrow$ score $y$ with $M$ under $(q,\varnothing)$\;
\For{$k \leftarrow 1$ \KwTo $K$}{
  $\ell^{(k)}, P^{(k)} \leftarrow$ score $y$ with $M$ under $(q, c\setminus s_k)$\;
}
compute per-token gap, $\textsf{jsd}_{\varnothing}$, leave-one-out drop, $\textsf{jsd}_{\mathrm{loo}}$ from Eqs.~\eqref{eq:gap}--\eqref{eq:jsdloo}\;
\ForEach{sentence span $y'$ in $y$}{
  aggregate per-token features over the tokens of $y'$\;
  $k^{*} \leftarrow \arg\max_k$ likelihood drop on $y'$\;
  emit features of $y'$ and supporting unit $s_{k^{*}}$\;
}
\caption{Grounding-sensitivity scoring of a RAG answer.}
\label{alg:gs}
\end{algorithm}
\begin{table}[H]
\centering
\caption{Computational cost per response, with answer length $A$, vocabulary $V$, and $K$ context chunks.}
\label{tab:cost}
\small
\begin{tabular}{lcc}
\toprule
Condition & Forward passes & Divergence cost \\
\midrule
Full context & 1 & n/a \\
\rowsep
No context & 1 & $O(AV)$ \\
\rowsep
Leave-one-out & $K$ & $O(KAV)$ \\
\midrule
Total & $K+2$ & $O(KAV)$ \\
\bottomrule
\end{tabular}
\end{table}
\begin{table}[H]
\centering
\caption{Summary of the six grounding features and the baselines. Gap (two-pass) uses the gap feature alone;
Sensitivity (trained) and Sensitivity + baselines use all six. The two-pass features (gap, $\textsf{jsd}_{\varnothing}$) need no leave-one-out
passes.}
\label{tab:featdef}
\small
\begin{tabular}{ll}
\toprule
Name & Definition \\
\midrule
$\textsf{gap}$ & full vs no-context log-likelihood, Eq.~\eqref{eq:gap} \\
\rowsep
$\textsf{jsd}_{\varnothing}$ & full vs no-context JSD, Eq.~\eqref{eq:jsdn} \\
\rowsep
max drop & maximum leave-one-out drop, Eq.~\eqref{eq:drop} \\
\rowsep
mean drop & mean leave-one-out drop, Eq.~\eqref{eq:drop} \\
\rowsep
top-2 drop & mean of the two largest leave-one-out drops \\
\rowsep
max jsd & maximum leave-one-out JSD, Eq.~\eqref{eq:jsdloo} \\
\midrule
perplexity & Eq.~\eqref{eq:ppl} under full context \\
\rowsep
length & number of answer tokens \\
\rowsep
NLI entailment & entailment probability, context vs answer \\
\bottomrule
\end{tabular}
\end{table}
\begin{table}[H]
\centering
\caption{Notation used throughout the paper.}
\label{tab:notation}
\small
\begin{tabular}{@{}lp{5.2cm}@{}}
\toprule
Symbol & Meaning \\
\midrule
$q,\,c,\,y$ & query, retrieved context, generated answer \\
\rowsep
$s_k$ & the $k$-th context chunk, $c=\bigcup_k s_k$ \\
\rowsep
$c\setminus s_k$ & context with chunk $k$ removed \\
\rowsep
$\varnothing$ & empty (no) context \\
\rowsep
$K$ & number of context chunks \\
\rowsep
$y^{j},\,m$ & the $j$-th answer span and the number of spans in $y$ \\
\rowsep
$|y|$ & number of tokens in span $y$ \\
\midrule
$Y^{j},\,Y$ & span and response hallucination labels, $Y=\max_j Y^{j}$ \\
\rowsep
$\tau^{j}$ & hallucination type of span $j$ (baseless, conflict, none) \\
\rowsep
$f$ & hallucination score; a higher value is more suspicious \\
\midrule
$\ell^{c},\ell^{\varnothing},\ell^{(k)}$ & answer log-likelihood under $c$, under no context, and with chunk $k$ removed \\
\rowsep
$p^{c}_t$ & predictive distribution at step $t$ under $c$ \\
\rowsep
$\textsf{jsd}_{\varnothing},\textsf{jsd}_{\mathrm{loo}}$ & no-context and leave-one-out divergence features \\
\rowsep
$A,\,V$ & answer length; vocabulary size \\
\midrule
$p^{c\setminus s}_t$ & predictive distribution at step $t$ with unit $s$ removed \\
\rowsep
$\mathrm{JSD}$ & Jensen-Shannon divergence \\
\bottomrule
\end{tabular}
\end{table}
\begin{table}[t]
\centering
\caption{Configuration used for all reported experiments.}
\label{tab:hyper}
\small
\begin{tabularx}{\columnwidth}{@{}>{\raggedright\arraybackslash}p{0.34\columnwidth}X@{}}
\toprule
Component & Setting \\
\midrule
Scorer models & Qwen2.5-1.5B / 7B / 14B, SmolLM2-1.7B, Llama-3.1-8B, Phi-3.5-mini \\
\rowsep
Advanced verifiers & chunk-NLI (deberta-v3-large), AlignScore, MiniCheck, ContextCite, LLM judge (14B) \\
\rowsep
Max context tokens & 1800 \\
\rowsep
Max answer tokens & 256 \\
\rowsep
Context chunks $K$ & 5 \\
\rowsep
Classifier & logistic regression, class-balanced, $L_2$, $C{=}1.0$, lbfgs, max\_iter $1000$ \\
\rowsep
Evaluation split & source-level, 60/40 \\
\rowsep
Bootstrap resamples & 2000, source-level \\
\rowsep
Random seed & 42 (data sampling), 0 (split and bootstrap) \\
\midrule
GPU & NVIDIA L40S, 48\,GB \\
\midrule
Python & 3.11 \\
\rowsep
PyTorch & 2.4.0 (CUDA 12.4) \\
\rowsep
transformers & 4.44.2 \\
\rowsep
Hugging Face datasets & 2.21.0 \\
\rowsep
scikit-learn & 1.5 \\
\bottomrule
\end{tabularx}
\end{table}

\section{Truncation Audit}\label{app:trunc}
At the $1800$-token context and $256$-token answer caps, RAGTruth retains $99.9\%$ of the context and $98.2\%$ of the answer on average with the $14$B scorer, $12.3\%$ of responses lose some content, and $13$ annotated hallucination spans start beyond the retained answer and are excluded. TofuEval is fully retained. RAGBench retains $94.9\%$ of the context and $97.1\%$ of the answer on average but cuts its longest contexts to as little as $11.7\%$, with $23.4\%$ of cases partly truncated and $72$ hallucination spans outside the scored window. A chunk that falls outside the window is recorded as missing, so the leave-one-out features stay aligned with the chunks actually scored. The effect of the excluded spans was not isolated experimentally, and the heavier truncation is a caveat on the RAGBench results.

\section*{Statements and Declarations}
\vspace{0.5em}
\begin{itemize}
  \item \textbf{Funding:} Not applicable.
  \item \textbf{Conflict of Interest:} The author declares no competing interests.
  \item \textbf{Availability of Data and Materials:} The three benchmarks (RAGTruth, TofuEval-MeetingBank, and RAGBench) are publicly available. The code that reproduces every table and figure is at \url{https://github.com/drbouke/GASP}, and the detector is installable with \texttt{pip install gasp-rag} (\url{https://pypi.org/project/gasp-rag/}).
  \item \textbf{Ethics Approval:} Not applicable.
\end{itemize}

{\def\bibfont{\normalfont\small}\small\setlength{\parskip}{0pt}\raggedright
}

\section*{Author Biography}
\noindent
\begin{minipage}[t]{2.6cm}
\vspace{0pt}
\includegraphics[width=2.5cm]{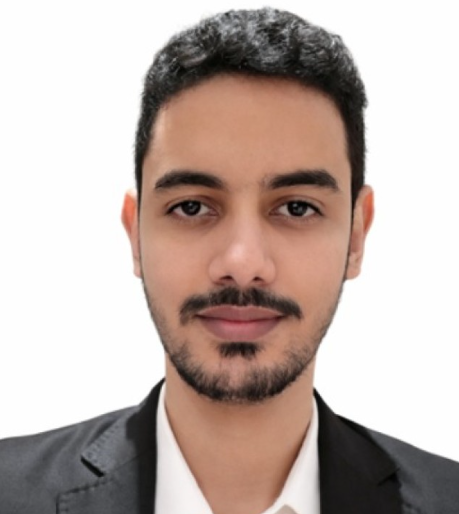}
\end{minipage}\hfill
\begin{minipage}[t]{\dimexpr\columnwidth-2.9cm\relax}
\vspace{0pt}
\textbf{Mohamed Aly Bouke} is a Postdoctoral Researcher at Multimedia University, Malaysia, and a Senior Member of IEEE.
He received his PhD in Information Security from Universiti Putra Malaysia. He serves as an Associate Editor for
\emph{e-Prime -- Nexus of Electrical, Electronic, and Intelligent Engineering}. His research interests include
cybersecurity, artificial intelligence, cryptography, and mathematical modeling.
\end{minipage}

\end{document}